\documentclass{article}

\usepackage{microtype}
\usepackage{ bbold}
\usepackage{graphicx}
\usepackage{subfigure}
\usepackage{booktabs} %
\usepackage{amsmath}
\usepackage{enumitem}
\usepackage{apxproof}
\usepackage{graphicx} %
\usepackage{amsmath} %

\usepackage{graphicx}
\usepackage[colorlinks=true, allcolors=blue]{hyperref}
\usepackage{mystyle}
\usepackage{amsmath}
\usepackage{xcolor}  
\colorlet{michaels}{violet}
\newcommand{\michaelcomment}[1]{}

\newcommand{\camera}[1]{\textcolor{black}{#1}}
\newcommand{\cij}[1]{$c^{i,j}_{\alpha, \beta}$}
\newcommand{\A}{\mathtt{A}}
\newcommand{\B}{\mathtt{B}}
\newcommand{\C}{\mathtt{C}}

\usepackage{hyperref}

\usepackage[accepted]{icml2024}

\usepackage{amsmath}
\usepackage{amssymb}
\usepackage{mathtools}
\usepackage{amsthm}

\usepackage[capitalize,noabbrev]{cleveref}

\theoremstyle{plain}

\theoremstyle{definition}

\theoremstyle{remark}

\usepackage[textsize=tiny]{todonotes}

\icmltitlerunning{{How do Transformers Perform In-Context Autoregressive Learning?}}
\begin{document}

\twocolumn[
\icmltitle{How do Transformers Perform In-Context Autoregressive Learning?}

\icmlsetsymbol{equal}{*}

\begin{icmlauthorlist}

\icmlauthor{Michaël E. Sander}{ens}
\icmlauthor{Raja Giryes}{tlv}
\icmlauthor{Taiji Suzuki}{tok}
\icmlauthor{Mathieu Blondel}{gdm}
\icmlauthor{Gabriel Peyré}{ens}

\end{icmlauthorlist}

\icmlaffiliation{ens}{Ecole Normale Sup\'erieure and CNRS, France}
\icmlaffiliation{tlv}{Tel Aviv University, Israel}
\icmlaffiliation{tok}{University of Tokyo and RIKEN AIP, Japan}
\icmlaffiliation{gdm}{Google DeepMind}

\icmlcorrespondingauthor{}{michael.sander@ens.fr}

\icmlkeywords{Machine Learning, ICML}

\vskip 0.3in

]

\printAffiliationsAndNotice{} %

\begin{abstract}
Transformers have achieved state-of-the-art performance in language modeling tasks. 
However, the reasons behind their tremendous success are still unclear. 
In this paper, towards a better understanding,
we train a Transformer model on a simple next token prediction task, where sequences are generated as a first-order autoregressive process $s_{t+1} = W s_t$. We show how a trained Transformer predicts the next token by first learning $W$ in-context, and then applying a prediction mapping. We call the resulting procedure \textit{in-context autoregressive learning}.
More precisely, focusing on commuting orthogonal matrices $W$, we first show that a trained one-layer linear Transformer implements one step of gradient descent for the minimization of an inner objective function when considering augmented tokens. 
When the tokens are not augmented, we characterize the global minima of a one-layer diagonal linear multi-head Transformer.
Importantly, we exhibit orthogonality between heads and show that positional encoding captures trigonometric relations in the data.
On the experimental side, we consider the general case of non-commuting orthogonal matrices and generalize our theoretical findings.
\end{abstract}
\section{Introduction}

Transformers \citep{vaswani2017attention} have achieved state-of-the-art performance in natural language processing tasks \citep{devlin2018bert}. They now serve as the backbone for large language models, such as GPT \citep{radford2018improving, brown2020language}, Chinchilla \citep{hoffmann2022training}, PaLM \citep{ chowdhery2023palm}, LLama \citep{touvron2023llama} or Mistral \citep{jiang2023mistral}. 
These models, which are causal, are trained to predict the next token $s_{T+1}$ given a sequence (also termed as context) $s_{1:T} \eqdef (s_1, \cdots, s_{T})$. 
An intriguing property of large Transformers is their ability to adapt their computations given the context $s_{1:T}$.  
In this work, we make a step towards understanding this \textit{in-context learning ability}. More precisely, assuming the tokens satisfy a relation $s_{T+1} = \phi_W(s_{1:T})$, with $W$ a context-dependent parameter varying with each sequence, we say that a trained Transformer \textit{autoregressively learns} this relation \textit{in-context} if it decomposes its prediction into 2 steps: first, estimating $W$ through an in-context mapping, and then applying a simple prediction mapping, which is equal or closely related to $\phi_W$ (see Definition \ref{def:ICL}).
The goal of this paper \camera{is to fully characterize the autoregressive in-context learning process for optimally-trained Transformers}.
More precisely, building on the work of \citet{vonoswald2023uncovering}, we focus on a simple autoregressive (AR) process of order $1$, where each sequence is generated following the recursion $s_{T+1} = \phi_W(s_{1:T}) \eqdef W s_T$ 
, and $W$ is a randomly sampled orthogonal matrix, referred to as the \textit{context matrix}. Such a process is illustrated in dimension $3$ in \autoref{fig:intro} for two different matrices $W$. 
We investigate the training of a linear Transformer to predict the next token in these AR processes, examining how it estimates $W$ in-context and makes predictions for $s_{T+1}$. 
Depending on the input tokens encoding, the in-context mapping can correspond to gradient descent on an inner objective, as suggested by \citet{vonoswald2023uncovering}. Alternatively, the context matrix $W$ might be determined in closed form if the model possesses sufficient expressiveness. This paper investigates both scenarios.

\begin{figure}[H]
\centering
\includegraphics[width=0.75\columnwidth]{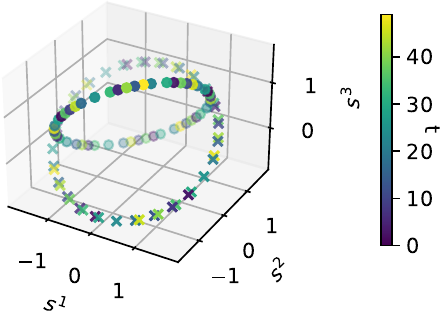} 
\caption{\textbf{Illustration of the autoregressive process in $\RR^3$.} Dots and crosses correspond to two different orthogonal matrices $W$. }\label{fig:intro}
\end{figure}

\paragraph{More precisely, we make the following contributions:}
\begin{itemize}[topsep=2pt,itemsep=2pt,parsep=2pt,leftmargin=10pt]
\item We begin by reviewing the background and previous works in \S \ref{sec:background}. Then, in \S \ref{sec:method}, we introduce our autoregressive process, which allows us to mathematically formalize the notion of \textit{in-context autoregressive learning}.

\item In \S \ref{sec:aug}, we demonstrate that if the matrices $W$ commute and the model parameters possess a block structure, then a linear Transformer—trained on augmented tokens as introduced by \citet{vonoswald2023uncovering}—effectively implements a step of gradient descent on an underlying objective function as in-context mapping. %

\item In \S \ref{sec:non_aug},  we turn our attention to a one-layer linear attention Transformer that incorporates positional encoding but does not use augmented tokens. We comprehensively characterize the minimizers of the training loss. Notably, these minimizers display an orthogonality property across different heads. This aspect underscores the significance of positional encoding in enabling the Transformer to learn geometric operations between tokens through its in-context mapping. 
We also study positional-encoding-only attention and show that approximate minimum $\ell_2$ norm solutions are favored by the optimization process. 

\item On the experimental side, in \S \ref{sec:exp}, we extend our analysis to the more general case where the context matrices $W$ do not commute. We validate our theoretical findings for both augmented and non-augmented scenarios. 
Furthermore, we explore how variations in the distribution of the context matrices $W$ affect trained positional encodings and lead to structures resembling those of traditional positional encodings commonly used in Transformers. 
\end{itemize}

\camera{Each theoretical result of the paper aims at characterizing the autoregressive in-context learning mechanism for simple models and sequence data. Namely, Propositions \ref{prop:gd}, \ref{prop:optimality_unitary} and \ref{prop:optimality_ortho} give the structure of the minimizers of the training loss and explicit the corresponding in-context mappings, while Propositions \ref{prop:quadra_loss}, \ref{prop:1d} and \ref{prop:pe} focus on the optimization process.} 
\paragraph{Notations.} We use lower cases for vectors and upper cases for matrices. $\| .\|$ is the $\ell_2$ norm. 
We denote the transpose and adjoint operators by $^\top$ and $^\star$. $O(d)$ (resp $U(d)$) is the orthogonal (resp unitary) manifold, that is $O(d) = \{W \in \RR^{d \times d} | W^\top W = I_d\}$ and $U(d) = \{W \in \CC^{d \times d} | W^\star W = I_d\}$. 
The element-wise multiplication is $\odot$. $\langle | \rangle$ is the canonical dot product in $\RR^d$, and $\langle \cdot|\cdot \rangle_{\CC}$ the canonical hermitian product in $\CC^d$. For $\lambda \in \CC^d$, $\lambda^k$ is the element-wise power $k$ of $\lambda$: $(\lambda^k)_i = \lambda^k_i$. 
\section{Background and previous works}\label{sec:background}

\paragraph{Causal Language Modelling.}

Language (or sequence) modeling refers to the development of models to predict the likelihood of observing a sequence $(x_1, \dots, x_T)$, where each $x_t$ is called a token, and comes from a finite vocabulary. 
This can be done by using the chain rule of probability $P(X_1 = x_1, X_2 = x_2, \ldots, X_T = x_T) = P(X_1 = x_1) \times P(X_2 = x_2 | X_1=x_1) \times  \ldots \times P(X_T = x_T | X_1 = x_1, \ldots, X_{T-1} = x_{T-1})$ \citep{Jurafsky2009}. 
Predicting these conditional probabilities can be done using a parametrized model $\mathcal{F}_{\theta}$ to minimize the loss $L(\mathcal{F}_{\theta}(x_1,...,x_{T-1}), x_{T})$ across all training samples and sequence length $T$.
In common applications, $L$ is chosen as the cross-entropy loss. In other words, the model is trained to predict the next token sequentially.
Such a model is called a causal language model: it cannot access future tokens.
Recently, the Transformer has emerged as the model of choice for language modeling.
\paragraph{Transformers.}

Transformers \citep{vaswani2017attention} process sequences of tokens $(x_1, \dots, x_T)$ of arbitrary length $T$. In its causal form \citep{brown2020language,touvron2023llama, jiang2023mistral},   
a Transformer first embeds the tokens to obtain a sequence $(e_1, \dots, e_T)$.
It is then composed of a succession of blocks with residual connections \citep{he2016deep}.
Each block is made of the composition of a multi-head self-attention
module and a multi-layer perceptron (MLP). 
Importantly, the latter acts on each token separately, whereas multi-head self-attention mixes tokens, and corresponds to applying vanilla self-attention in parallel \citep{michel2019sixteen}. 
More precisely, each multi-head self-attention is parametrized by a collection of weight matrices $(W^h_Q, W^h_K, W^h_V, W^h_O)_{1\leq h \leq H}$ and returns:
\begin{equation}\label{eq:MHSA}
(\sum_{h=1}^{H} W^h_O\sum_{t'=1}^{t} \mathcal{A}^h_{t,t'} W^h_V e_{t'})_{t\in \{1, \cdots, T\}},
\end{equation}
where $\mathcal{A}^{h}$ is the attention matrix \citep{bahdanau2014neural} and is usually defined as 
$$
\mathcal{A}^{h}_{t,:}=\mathrm{softmax}(\langle W^h_Q e_{t} , W^h_Ke_{:}\rangle ),
$$
with $\langle \cdot , \cdot \rangle $ a dot product. 
The sum over $t'$ in \eqref{eq:MHSA} stopping at $t$ reflects the causal aspect of the model: the future cannot influence the past. The output at position $T$ is commonly used to predict the next token $e_{T+1}$. 
In practice, to help the model encode the relative position of the tokens in the sequence, a positional encoding (PE) is used.  

\paragraph{Positional encoding.}

As described in \citet{kazemnejad2023impact}, encoding the position in Transformers amounts to defining the dot product $\langle \cdot,\cdot \rangle $ in the attention matrix, using additional (learnable or not) parameters. 
Popular designs include Absolute PE \citep{vaswani2017attention}, Relative PE \citep{raffel2020exploring}, AliBI \citep{press2021train}, Rotary \citep{su2024roformer}, and NoPE \citep{kazemnejad2023impact}. In this paper, we consider a learnable positional encoding.

\paragraph{Linear attention.}
In its simplest form, linear attention \citep{katharopoulos2020transformers} consists in replacing the $\mathrm{softmax}$ in \eqref{eq:MHSA} by the identity. 
More formally, it consists in considering that each coefficient in the attention matrix is $\mathcal{A}^h_{t,t'} = \langle W^h_Ke_{t'},W^h_Qe_t\rangle $. 
The main practical motivation of linear attention is that it enables faster inference \citep{katharopoulos2020transformers, fournier2023practical}. Note that even though they are called linear Transformers, the resulting models are non-linear with respect to the input sequence and jointly non-linear with respect to the parameters.
From a theoretical perspective, linear attention has become the model of choice to understand the in-context-learning properties of Transformers \citep{mahankali2023one, ahn2023transformers, zhang2023trained}.

\paragraph{In-context-Learning in Transformers}

The seminal work of \citet{brown2020language} reported an in-context-learning phenomenon in Transformer language models: these models can solve few-shot learning problems given examples in-context. Namely, given a sequence $(x_1, f(x_1), x_2, f(x_2), \cdots, x_{n})$, a trained Transformer can infer the next output $f(x_n)$ without additional parameter updates. 
This surprising ability has been the focus of recent research.
Some works consider the $\mathrm{softmax}$ attention without considering training dynamics \citep{garg2022can, akyurek2022learning, li2023transformers}. 
Other works focus solely on linear attention and characterize the minimizers of the training loss when $f$ is sampled across linear forms on $\RR^d$, that is $f(x) = w^\top x$ for some $w$ \citep{mahankali2023one, ahn2023transformers, zhang2023trained}.
In particular, these works discuss the ability of Transformers to implement optimization algorithms in their forward pass at inference, as empirically suggested by \citet{von2023transformers}.
Nevertheless, the formulations used by \citet{von2023transformers, mahankali2023one, ahn2023transformers, zhang2023trained} are all based on concatenating the tokens so that the Transformer's input takes the form 
$\begin{psmallmatrix}
x_1 & x_2 & \cdots & x_n \\
f(x_1) & f(x_2) & \cdots & 0 
\end{psmallmatrix}
\in \mathbb{R}^{(d+1) \times n}.$ However, the necessity for this concatenation limits the impact of these results as there is no guarantee that the Transformer would implement this operation in its first layer. 
In addition, these works explicitly consider the minimization of an in-context loss, which is different from the next-token prediction loss in causal Transformers.
In contrast, our work considers the next-token prediction loss and considers a more general notion of in-context learning, namely \textit{in-context autoregressive learning}, that we describe in the next section.

\section{Linear Attention for  AR Processes}\label{sec:method}
\paragraph{Token encoding.}
Building on the framework established by \citet{vonoswald2023uncovering}, we consider a noiseless setting where each sequence begins with an initial token $s_1 = 1_d$. This token acts as a start-of-sentence marker. The subsequent states are generated according to $s_{t+1} = Ws_t$, where $W$ is a matrix referred to as the \textit{context matrix}. This matrix is sampled uniformly from a subset $\Cc_O$ (respectively, $\Cc_U$) of $O(d)$ (respectively, $U(d)$), and we denote $\mathcal{W}$ as the corresponding distribution: $W \sim \mathcal{W} \eqdef \Uu(\Cc)$.
Considering norm-preserving matrices ensures the stability of the AR process, which is crucial to be able to learn from long sequences (i.e. using large $T$). In this paper, we contrast $\Cc_O$ and $\Cc_U$ to showcase how the distribution of in-context parameters $W$ impacts the in-context mapping learned by Transformers. In addition, we have the following. 

\begin{rem}\label{rem:unit_ortho}
   If $W_U \in U(d)$, then
\[
W_O \eqdef \begin{bmatrix}
    \mathrm{Real}(W_U) & -\mathrm{Imag}(W_U) \\
    \mathrm{Imag}(W_U) & \mathrm{Real}(W_U)
\end{bmatrix}
\]
is in $O(2d)$ and has pairwise conjugate eigenvalues. $W_O$ is a \textit{rotation} (because $W_O$ is similar to a $2 \times 2$ block diagonal matrix with rotations). Reciprocally, for any rotation $W_O$ of size $2d$ corresponds a unitary matrix $W_U$ of size $d$ by selecting half of the eigenvalues (for instance those with positive imaginary parts).
\end{rem}
 
Therefore, $U(d)$ can be viewed as a subset of $O(2d)$, while $O(d)\subset U(d)$. As such, placing ourselves in $U(d)$ corresponds to a compact way of considering real AR processes in dimension $2d$.

In our analysis, we consider two settings in which the sequence $s_{1:T}$ is mapped to a new sequence $e_{1:T}$. In the \textit{augmented setting} (\S \ref{sec:aug}), the tokens are defined as $e_t \eqdef (0, s_t, s_{t-1})$, aligning with the setup used by \citet{vonoswald2023uncovering}. In contrast, the \textit{non-augmented setting} (\S \ref{sec:non_aug}) utilizes a simpler definition where the tokens are simply $e_t \eqdef s_t$.

\paragraph{Model and training process.}

We consider a  Transformer 
with linear attention, which includes an optionally trainable positional encoding $P \in \RR^{T_\mathrm{max} \times T_\mathrm{max}}$ for some $T_\mathrm{max} \in \NN$:
\begin{equation}\label{eq:linear_att}
    \mathcal{A}^h_{t,t'} = P_{t,t'} \langle W^h_Q e_{t} | W^h_Ke_{t'}\rangle.
\end{equation}
Throughout this paper, we will re-parameterize the model by setting $B^h = W^h_OW^h_V$ and $A^h = W_K^{h\top} W^h_Q$.
Note that such an assumption is standard in theoretical studies on the training of Transformers \citep{mahankali2023one, zhang2023trained, ahn2023transformers}.
The trainable parameters are therefore $\theta=((A^h,B^h,P))_{1\leq h \leq H}$ when the positional encoding is trainable and $\theta = ((A^h,B^h))_{1\leq h \leq H}$ otherwise. This defines a mapping $\mathcal{T}_{\theta}(e_{1:T})$ by selecting a section from some element $\tau$ in the output sequence \eqref{eq:MHSA}.
We focus on the population loss, defined as:
\begin{equation}\label{eq:optim}
\ell (\theta) \eqdef \sum_{T=2}^{T_{\mathrm{max}}} \EE_{W \sim \mathcal{W}} \| \mathcal{T}_{\theta}(e_{1:T}) - s_{T+1}  \|^2, 
\end{equation}
indicating the model's objective to predict $s_{T+1}$ given $e_{1:T}$. It is important to note that both $s_{T+1}$ and $e_{1:T}$ appearing in~\eqref{eq:optim} are computed from a random $W$ and are therefore random variables.
\paragraph{In-context autoregressive learning.}
Our goal is to theoretically characterize the parameters $\theta^{*}$ that minimize $\ell$, discuss the convergence of gradient descent to these minima, and characterize the in-context autoregressive learning of the model. This learning process is defined as the model's ability to learn and adapt within the given context: first by estimating $W$ (or more generally some power of $W$) using an in-context mapping $\Gamma$, then by predicting the next token using a simple mapping $\psi$. In the context or AR processes, we formalize this procedure in the following definition.
\begin{defn}[In-context autoregressive learning]\label{def:ICL}
We say that $\mathcal{T}_{\theta^*}$ \emph{learns autoregressively in-context} the AR process $s_{T+1} = Ws_T$ if $\mathcal{T}_{\theta^*}(e_{1:T})$ can be decomposed in two steps: (1) first applying an in-context mapping $\ga = \Gamma_{\theta^*}(e_{1:T})$, (2) then using a prediction mapping
$\mathcal{T}_{\theta^*}(e_{1:T}) = \psi_\gamma(e_{1:T})$.
This prediction mapping should be of the form 
$\psi_{\gamma}(e_{1:T}) = \gamma s_{\tau}$ for some shift $\tau \in \{1, \cdots, T \}$.
With such a factorization, in-context learning arises when the training loss $\ell(\theta^*)$ is small. This corresponds to having $\Gamma_{\theta^*}(e_{1:T}) \approx W^{T+1 - \tau}$  when applied to data $e_{1:T}$ exactly generated by the AR process with matrix $W$.
\end{defn}
In this work, we will have either $\tau=T$ or $\tau = T-1$. 
\begin{rem}
    \camera{We use the word \textit{in-context} to make explicit the fact that the matrix $W$ is different for each sequence. As a consequence, attention-based models are particularly well suited to such a task because predicting $W$ involves considering relationships between tokens. In contrast, RNNs perform poorly in this setting precisely because they do not consider interactions between tokens. In fact, in its simplest form, a linear RNN with parameters $A$ and $B$ outputs, for each $t$: 
$y_t = \sum_{k=1}^t A^{t-k} B W^{k-1} e_1$. It is easy to see that $A$ and $B$ would have to depend on $W$ for $y_t$ to be close to $W^t s_1$, which is impossible because $W$ is different for each sequence.}
\end{rem} 

To fully characterize the in-context mapping $\Gamma$ and prediction mapping $\psi$, we rely on a commutativity assumption. 
\begin{asp}[Commutativity]\label{asp:commute}
The matrices $W$ in $\Cc$ commute. Hence, they are co-diagonalizable in a unitary basis of $\CC^{d \times d}$. Up to a change of basis, we therefore suppose $\Cc_U = \{\mathrm{diag}(\lambda_1, \cdots, \lambda_d) ,  |\lambda_i|=1 \}$ and $\Cc_{O} =  \{(\lambda_1, \Bar{\lambda_1}, \cdots, \lambda_{\delta}, \Bar{\lambda_{\delta}}),|\lambda_i|=1\}$, with $d = 2 \delta$.
\end{asp}

For conciseness, we only consider pairwise conjugate eigenvalues in $\Cc_{O}$. While assumption \ref{asp:commute} is a strong one, it is a standard practice in the study of matrix-involved learning problems \citep{arora2019implicit}.
We highlight that, to the best of our knowledge, this is the first work that provides a theoretical characterization of the minima of $\ell$. 
The general problem involving non-commutative matrices is complex, and we leave it for future work.
Note that recent studies, such as those by \citet{mahankali2023one, zhang2023trained, ahn2023transformers} focus on linear regression problems $x \mapsto w^\top x$, which rewrites $x \mapsto  1^\top _d\mathrm{diag}(w)x$. Therefore, considering diagonal matrices is a natural extension of these approaches to autoregressive settings.
Note also that imposing commutativity, while a simplification, represents a practical method of narrowing down the class of models $\phi_W$. Indeed, in high dimension, it becomes necessary to restrict the set $\Cc$, otherwise, $W$ cannot be accurately estimated when $T < d$.
Note that however, in \S \ref{sec:exp}, we experimentally consider the general case of non-commuting matrices.

\section{In-context mapping with gradient descent}\label{sec:aug} 

In this section, we consider the augmented tokens $e_t \eqdef (0, s_t, s_{t-1}).$
We show that under assumption \ref{asp:commute} and an additional assumption on the structure of $\theta$ at initialization, the minimization of \eqref{eq:optim} leads to the linear Transformer implementing one-step of gradient descent on an inner objective as its in-context mapping $\Gamma_{\theta^*}$. The motivation behind this augmented dataset is that the tokens $e_t$ can be computed after a two-head self-attention layer with a $\mathrm{softmax}$.
Indeed, we have the following result. 
\begin{lem}\label{lem:augmented_heads}
The tokens $e_{1:T}$ can be approximated with arbitrary precision given tokens $s_{1:T}$ with a Transformer \eqref{eq:MHSA}. 
\end{lem}
For a proof, see Appendix \ref{proof:augmented_heads}. 
We suppose that $\mathcal{W} =\Uu(\Cc_U)$, that is we consider unitary matrices. We consider $\mathcal{T}_{\theta}$ to be a one-head attention layer with skip connection and output the fist $d$ coordinates of the token $T$. 
More precisely, one has 
\begin{equation}\label{eq:augmented_att}
     \mathcal{T}_{\theta}(e_{1:T}) = \left( e_T +  \sum_{t=1}^T \langle A e_T | e_{t}  \rangle_{\CC} B e_{t} \right)_{1:d}.
\end{equation}
Importantly, we do not consider a positional encoding as the relative position is already stored in each token $e_t$. Note that we use the hermitian product $\langle | \rangle_{\CC}$ as $e_t \in \CC^{3d}$.

We have the following result showing the existence of $\theta_0$ such that \eqref{eq:augmented_att} corresponds to one step of gradient descent on an inner objective.
\begin{prop}[\citet{vonoswald2023uncovering}]\label{prop:gd_von}
There exists $\theta_0$ such that $\Gamma_{\theta_0}(e_{1:T}) = W_0 - \eta  \nabla L(W_0, e_{1:T})$
with 
\begin{equation}\label{eq:inner}
L(W, e_{1:T}) =\frac12 \sum_{t=1}^{T-1} \|s_{t+1} - Ws_{t}\|^2,
\end{equation}
and $W_0$ is any gradient descent initialization.
\end{prop}

We now make the following assumption on the structure of $A$ and $B$.
\begin{asp}\label{asp:init}
We parameterize $A$ and $B$ as
\[
\scalebox{1.}{ %
\(
A =\left(\begin{array}{ccc}
0 & 0 & 0 \\
0 & A_1 & A_2\\
0 & A_3 & A_4
\end{array}\right)
\quad \mathrm{and} \quad B = \left(\begin{array}{ccc}
0 & B_1& B_2 \\
0 & 0 & 0 \\
0 & 0 & 0
\end{array}\right),
\)
}
\]
with 
$A_i = a_i I$ and $B_i = b_i I$.
\end{asp}
Importantly, while the zero block structure is stable with gradient descent on loss \eqref{eq:optim}, we do impose the non-zero blocks to stay diagonal during training.
Note that considering diagonal matrices is a widely used assumption in the topic of linear diagonal networks \citep{woodworth2020kernel, pesme2021implicit}. Note also that we consider the general parametrization for $A$ and $B$ in our experiments in \S \ref{sec:exp}. 

Under assumption \ref{asp:init}, we have the following result, stating that at optimality, $\Gamma_{\theta^*}$ corresponds to $\Gamma_{\theta_0}$ in Proposition \ref{prop:gd} with $W_0 = 0$.

\begin{prop}[In-context autoregressive learning with gradient-descent]\label{prop:gd}
Suppose  $\Cc = \Cc_U$, assumptions \ref{asp:commute} and \ref{asp:init}.
Then loss \eqref{eq:optim} is minimal for $\theta^*$ such that $a^*_1 + a^*_4 = a^*_2 = b^*_2 = 0$ and $a^*_3 b^*_1 = \frac{\sum_{T=2}^{T_{\mathrm{max}}}T}{\sum_{T=2}^{T_{\mathrm{max}}}(T^2 + (d-1)T)}$.
Furthermore, an optimal in-context mapping $\Gamma_{\theta^*}$ is one step of gradient descent starting from the initialization $\lambda = 0$, with a step size asymptotically equivalent to $\frac{3}{2 T_{\max}}$ with respect to $T_{\mathrm{max}}$.
\end{prop}
For a full proof, see Appendix \ref{proof:gd}. Proposition \ref{prop:gd} demonstrates that a single step of gradient descent constitutes the optimal forward rule for the Transformer $\mathcal{T}_{\theta}$. This finding aligns with recent research showing that one step of gradient descent is the optimal in-context learner for one layer of self-attention in the context of linear regression \citep{mahankali2023one, zhang2023trained, ahn2023transformers}.
\camera{However, a substantial drawback of considering augmented tokens is that it requires previous layers to form these tokens which—although is possible according to Lemma \ref{lem:augmented_heads}—is a strong assumption. Therefore, in the next section, we consider the non-augmented setting where we do not make strong assumptions about previous layers.
} 
\section{In-context mapping as a geometric relation}\label{sec:non_aug}

In this section, we consider the non-augmented tokens where $e_t \eqdef s_t$, and a multi-head self-attention model $\mathcal{T}_{\theta}$:
\begin{equation}\label{eq:att}
     \mathcal{T}_{\theta}(e_{1:T}) =\sum_{h=1}^H \sum_{t=1}^T P_{T-1,t}\langle e_{t}| A^h e_{T-1}\rangle_{\CC}  B^h  e_{t} ,
\end{equation}
that is we consider $\tau = T-1$, the second to last token in the output, and no residual connections. While not considering the last token in the output is not done in practice, this small modification is necessary to achieve zero population loss. We stress out that \textit{we still mask} the token we want to predict.

We consider a self-attention module with $H$ heads, and we define the following assumption.

\begin{asp}\label{asp:diag}
$A^h$ and $B^h$ are diagonal for all $h$: $A^h = \diag(a_h)$ and $B^h = \diag(b_h)$ with $(a_h, b_h) \in \RR^{d} \times \RR^d$.
\end{asp}
Importantly, we impose the diagonal structure during training. This diagonal aspect reflects the diagonal property of the context matrices. 
Under assumptions \ref{asp:commute} and \ref{asp:diag}, we have the following result.
\begin{lem}\label{lem:forward_heads}
Suppose assumptions \ref{asp:commute} and \ref{asp:diag}. Writing $a_h = (a^1_h, \cdots, a^d_h)$ and $b_h = (b^1_h, \cdots, b^d_h)$ and letting $\A \eqdef (a^1, \cdots, a^d) \in \RR^{H \times d}$, and $\B \eqdef (b^1, \cdots, b^d)$, one has for an input sequence $e_{1:T} = (1_d, \lambda, \cdots, \lambda^{T-1}) \in \RR^{T\times d}:$ 
$$
\mathcal{T}_{\theta}(e_{1:T}) = \sum^{T}_{t=1}P_{T-1,t} [\B^\top \A] \lambda^{t-T+1} \odot \lambda^{t-1}.
$$
\end{lem}
For a full proof, refer to appendix \ref{proof:forward_heads}.
Note that $\A$ and $\B$ correspond to the concatenation of diagonals $W_K^{h\top} W^h_Q$ and $W^h_OW^h_V$ along heads.

It is easily seen from Lemma \ref{lem:forward_heads} that the choice of $P_{T-1,t}^* = \delta_{t=T}$ and ${\B^*}^\top \A^* = I_d$ implies $\mathcal{T}_{\theta^*}(e_{1:T}) = \lambda^T$, and therefore $\ell(\theta^*) = 0.$ We see that this requires at least $d$ heads. 
A natural question is whether there are other optimal solutions and how to characterize them.
To answer this question, we investigate the case of unitary context matrices before moving to orthogonal ones. We consider these cases separately because they both lead to different in-context mappings. We recall that $U(d)$ can be seen as a subset of $O(2d)$ (see Remark \ref{rem:unit_ortho}).

\subsection{Unitary context matrices.}

In this case, coefficients in the context matrices are drawn independently. This constrains the possible values for $\theta^*$ achieving zero loss. Indeed, we have the following result.

\begin{prop}[Unitary optimal in-context mapping]\label{prop:optimality_unitary}
Suppose assumptions \ref{asp:commute} and \ref{asp:diag}.
Any $\theta^* =(\A^*,\B^*,P^*)$ achieving zero of the loss \eqref{eq:optim} satisfies $P^*_{T-1,t}=0$ if $t \neq T$, $P^*_{T-1,T} ({\B^*}^\top \A^*)_{ii} = 1$, and $({\B^*}^\top \A^*)_{ij} = 0$ for $i\neq j.$ 
Therefore, one must have $H \geq d$. 
An optimal in-context mapping satisfies $\Gamma_{\theta^*}(e_{1:T}) = \bar{e}_{T-1} \odot e_T$ and the predictive mapping $\psi_{\gamma}(e_{1:T}) = \gamma \odot e_T$.
\end{prop}
\begin{proofsketch}
At $\ell(\theta^*) = 0$ one has $\mathcal{T}_{\theta^*}(e_{1:T}) = \lambda^T$. We notice that $\mathcal{T}_{\theta^*}(e_{1:T})$ is a polynomial in the $\lambda_i$'s. Identifying the coefficients leads to the desired results. 
\end{proofsketch}
For a full proof, refer to appendix \ref{proof:optimality_unitary}. In particular, for $e_{1:T} = (1_d, \lambda, \cdots, \lambda^{T-1})$, we have $\Gamma_{\theta^*}(e_{1:T}) = \lambda$, and $\psi_{\Gamma_{\theta^*}}(e_{1:T}) =  \lambda \odot \lambda^{T -1} = \lambda^T$.
\paragraph{Orthogonality.}The equality $(\B^\top \A)_{ij} = 0$ for $i\neq j$ corresponds to an orthogonality property between heads. Indeed,
to further understand what Proposition \ref{prop:optimality_unitary} implies in terms of learned model, let's look at the particular case in which $H=d$ and, at optimality, $\A^* = \B^* = I_d$ and $P_{T-1,t}^* = \delta_{t=T}$.
Therefore, the positional encoding selects the last token in the input sequence, hence learning the structure of the training data.
In parallel, each attention matrix captures a coefficient in $\lambda$: 
$(\mathcal{T}_{\theta^*}(e_{1:T}))_h = \lambda_h \lambda_h^{T-1}$. When there are more than $d$ heads, some heads are useless, and can therefore be pruned. 
Such a finding can be related to the work of \citet{michel2019sixteen}, where the authors experimentally show that some heads can be pruned without significantly affecting the performance of Transformers. 
Orthogonality in the context of Transformers was also investigated by directly imposing orthogonality between the outputs of each attention head \citep{lee2019orthogonality} or on attention maps \citep{chen2022principle, zhang2021orthogonality}.
The ability of the positional encoding to recover the spatial structure was already shown by \citet{jelassi2022vision}, which studies Vision Transformers \citep{dosovitskiy2020image}.

\paragraph{Convergence of gradient descent.} 

Now that we have characterized all the global minima of the loss \eqref{eq:optim}, we can study the convergence of the optimization process.
We have the following Proposition, which shows that the population loss \eqref{eq:optim} writes as a quadratic form in $\B^\top \A$ and $P$, which enables connections with matrix factorization.
\begin{prop}[Quadratic loss]\label{prop:quadra_loss}
Under assumptions \ref{asp:commute} and \ref{asp:diag}, loss \eqref{eq:optim} reads 
$$
\ell (\A, \B, P) = \sum_{T=2}^{T_{\mathrm{max}}}l(\B^\top \A, P_{T-1})
$$
with $l(\C, p) = \|p\|^2_2 \|\C\|^2_F + p_{T-1}^2 S(\C^\top \C) - 2 \mathrm{Tr}(\C)p_{T} + d$, where $S$ is the sum of all coefficients operator. 
\end{prop}
A proof is in Appendix \ref{proof:quadra_loss}. For an optimal $P^*$ with $P_{T-1,t}^* = \delta_{t=T}$, $
\ell (\A, \B, P^*) = (T_{\mathrm{max}}-1)\| \B^\top \A - I\|^2_F, 
$ for which we can use Theorem 2.2 of \citet{nguegnang2021convergence} to argue that for almost all initial values, gradient flow on $\ell$ will converge to a global minimum, that is $\B^\top \A = I_d$. 
When training is also done on $p_T \eqdef P_{T-1,T}$, the loss is then $\ell(\A, \B, P) =\sum_{T=2}^{T_{\mathrm{max}}} \| p_T \B^\top \A - I\|^2_F$. Note that even for $T_{\mathrm{max}} = 2$, convergence of gradient descent in $(\A, \B, p_2)$ on $\ell$ to a global minimum is an open problem, for which a conjecture \citep{nguegnang2021convergence, achour2021loss} states that for almost all initialization, $(\A, \B, p_2)$ will converge to a global minimum of $\ell $. 
We provide evidence for global convergence in Figure \ref{fig:unit_orthp}.
Yet, we have the following result in the scalar case $H=d=1$. Its proof is in Appendix~\ref{proof:1d}.
\begin{prop}\label{prop:1d}
Consider the loss $\ell (a,b,p)=( p a b - 1)^2$. Suppose that at initialization, $|pab-1| < 1$. Then gradient flow on $(a,b,p)$ converges to a global minimum satisfying $a^*b^*p^* = 1$.
\end{prop}

\paragraph{Role of the $\mathbf{\mathrm{softmax}}$.}
    \camera{Our results rely heavily on the use of linear attention. In fact, we could not find a natural way to express the global minimum of the training loss \eqref{eq:optim} when a $\mathrm{softmax}$ layer was involved, even in dimension $d=1$. To gain more insight, we conducted an experiment where we trained different models with and without $\mathrm{softmax}$ and MLP layers. The results are shown in Figure \ref{fig:rebuttal_non_augmented_commuting} in Appendix \ref{app:exp}, where it is clear that in the case of commuting context matrices, using a $\mathrm{softmax}$ is incompatible with learning the underlying in-context mapping.}

\subsection{Orthogonal context matrices.}
We now turn to the case where the context matrices are in $\Cc_O$. We recall that this imposes that the $\lambda_i$ are pairwise conjugate. Therefore, the dimension $d$ is even, and we write it $d=2\delta$. The context matrices are therefore rotations. This property changes the optimization landscape and other solutions are possible, as shown in the following Lemma.
\begin{lem}\label{lem:trig}
Suppose assumptions \ref{asp:commute}, \ref{asp:diag}. If $P^*_{T-1,T-1}=-1$, $P^*_{T-1,T}=2$ and $0$ otherwise, and ${\B^*}^\top \A^* = \frac12 \diag(J, \cdots, J)$, with $J \in \RR^{2\times 2}$ and $J_{ij} = 1$ for all $i,j$, then $\mathcal{T}_{\theta^*}(e_{1:T}) = \lambda^T.$ 
\end{lem}
For a full proof, refer to Appendix \ref{proof:optimality_unitary}. In this case, $H = \delta $ heads are sufficient to reach zero population loss. 
The optimal parameters can be exactly characterized.
\begin{prop}[Orthogonal optimal in-context mapping]\label{prop:optimality_ortho}
Suppose assumptions \ref{asp:commute} and \ref{asp:diag}.
Any $ \theta^* = (\A^*,\B^*,P^*)$ with $\ell(\theta^*) = 0$ in \eqref{eq:optim} satisfies, denoting $\C^* = {\B^*}^\top \A^*$ and $p^* = P^*_{T-1}$:
$p^*_{t}=0$ if $t < T-1$, $p^*_{T} \C^*_{i,i} = 1$, $p^*_{T} \C^*_{2i-1,2i} + (\C^*_{2i-1,2i-1} + \C^*_{2i-1,2i})p^*_{T-1}= 0$, 
$p^*_{T} \C^*_{2i,2i-1} + (\C^*_{2i,2i} + \C^*_{2i,2i-1})p^*_{T-1}= 0$,
$\C^*_{2i-1, j} = \C^*_{2i, j}=0$ for $j \neq 2i-1, 2i$. An optimal in-context mapping is then, for $e_t = \lambda^{t-1}$: $\Gamma_{\theta^*}(e_{1:T}) = \lambda^2$, and the corresponding predictive mapping $\psi_{\Gamma_{\theta^*}(e_{1:T})}(e_{1:T}) = \lambda^2 \odot e_{T-1} = \lambda^{T}$.
\end{prop}
\begin{proofsketch}
Similarly to Proposition \ref{prop:optimality_unitary}, we identify the coefficients of a polynomial, with careful inspection of terms involving pairwise conjugate contexts $(\lambda_i, \bar{\lambda_i})$.
\end{proofsketch}
Similarly to Proposition \ref{prop:optimality_unitary}, this result indicates an orthogonal property between heads.
A closer look at the computation of $\Gamma_{\theta^*}$ reveals that the relation implemented in-context by the Transformer in Proposition \ref{prop:optimality_ortho} is an extension of a known formula in trigonometry: 
 $
 2 \cos{\theta}R_{\theta} - I_2 = R_{2\theta}
 $, with $R_{\theta}$ the rotation of parameter $\theta$ in $\RR^2$ (see Figure \ref{fig:trigo}). Importantly, when $\delta \leq H < 2\delta = d$, the optimal $\C^*$ in Proposition \ref{prop:optimality_ortho} is of rank $\delta$, which corresponds to Lemma \ref{lem:trig}. However, when $H  \geq d$, full-rank solutions are achievable.

 Under the assumptions of Proposition \ref{prop:optimality_unitary}, the population loss \eqref{eq:optim} is also a quadratic form in $P$ and $\B^\top \A$.
 Similarly to Proposition \ref{prop:quadra_loss}, global convergence results of gradient descent on such loss function are still an open problem \citep{achour2021loss}. We provide experimental evidence for convergence in Figure \ref{fig:unit_orthp}.

 \begin{figure}[H]
    \begin{minipage}[c]{0.55\linewidth}
        \includegraphics[width=\textwidth]{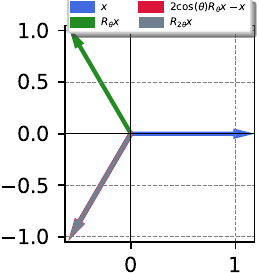}
    \end{minipage}\hfill
    \begin{minipage}[c]{0.4\linewidth}
\caption{\textbf{Trigonometric formula} implemented by the Transformer in-context. The minima of the training loss correspond to implementing, up to multiplying factors: $2 \cos{\theta}R_{\theta} - I_2 = R_{2\theta}.$}\label{fig:trigo}
    \end{minipage}
\end{figure} 

\begin{figure*}[ht]
 \centering
\includegraphics[width=1\textwidth]{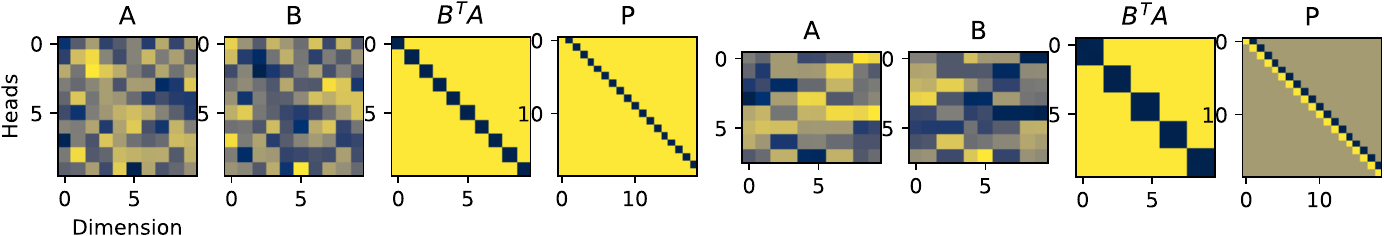} 
\vspace{-2em}
\caption{\textbf{Matrices $\A$, $\B$, $\B^\top \A$ and $P$ after training model \eqref{eq:att} on loss \eqref{eq:optim} with random initialization}. We take $d=10$ and $T=15$.
\textbf{Left:} Unitary context case with $H=10.$ 
\textbf{Right:} Orthogonal context case, with $H=8 < d$, which leads to low rank $\B^\top \A$. In both cases, we obtain arbitrarily small final loss. We recover parameters corresponding to our Propositions \ref{prop:optimality_unitary} and \ref{prop:optimality_ortho}.}
\label{fig:unit_orthp}
\vspace{-1em}
 \end{figure*}
 
\subsection{Positional encoding-only attention.}\label{subsec:pe_only}
We end this section by investigating the impact of the context distribution on the trained positional encoding $P$. For this, we consider a positional encoding-only Transformer, that is, we fix $\B^\top \A = I_d$. In this case, the problem decomposes component-wise, and we only need to consider the $d=1$ case. We therefore consider the AR process $s_{t+1} = \lambda s_t$ for $|\lambda|=1$. We break the symmetry of the context distribution: for $\mu \geq 1$ and $\theta \sim \Uu(0,2 \pi),$ we define $\lambda = e^{i \theta / \mu}$. We denote $\Ww(\mu)$ as the corresponding distribution. Therefore, we focus on the optimization problem: 
\begin{equation}\label{eq:pe}
    \min_{p \in \RR^T} l(p) \eqdef \EE_{\lambda \sim \Ww(\mu)} |\sum_{t=1}^T p_t \lambda^{2t-T} -\lambda^T|^2.
\end{equation}
Here again, the same proof as for Proposition \ref{prop:optimality_unitary} shows that the optimal positional encoding is $p^* = \delta_{t=T}$, meaning that we predict the next token using the last token in the context. However, depending on $\mu$, \eqref{eq:pe} can be ill-conditioned.
\begin{prop}[Conditioning]\label{prop:pe}
The Hessian $H \in \RR^{T\times T}$ of $l$ in \eqref{eq:pe} is given by 
$$H_{t, t’} = \frac{\mu}{4 \pi (t’-t)}\sin(4(t’-t)\frac{\pi}{\mu}).$$ Denoting $\sigma_1(\mu) \geq \dots \geq  \sigma_T(\mu)$ its eigenvalues, one has $\sigma_1(\mu) \to T$ and $\sigma_{t>1}(\mu) \to 0$ as $\mu \to + \infty$.
\end{prop}

Therefore, for large $\mu$, $H$ in Proposition \ref{eq:pe} is poorly conditioned. In such a setting, gradient descent, even with a large number of iterations, induces a $\ell_2$ regularization~\citep{yao2007early}. As an \camera{informal} consequence,  approximate solutions computed by gradient descent significantly deviate from the optimal $p^*$. As demonstrated experimentally in Figure \ref{fig:pe} and \S\ref{sec:exp}, the effect of this regularization is a spatial smoothing of the positional encoding, which leads to entirely different in-context mappings $\Gamma$, hence showing the effect of the optimization process on the in-context autoregressive learning abilities of Transformers.

\section{Experiments}\label{sec:exp}
In this section, we illustrate and extend our results through experiments. 
Our code in Pytorch \citep{paszke2017automatic} and JAX \citep{jax2018github} is open-sourced at \url{https://github.com/michaelsdr/ical}. 
We use the standard parametrization of Transformers, that is we train on $(W_Q, W_K, W_V, W_O)$.
\begin{figure}[H]
    \begin{minipage}[c]{0.55\linewidth}
        \includegraphics[width=\textwidth]{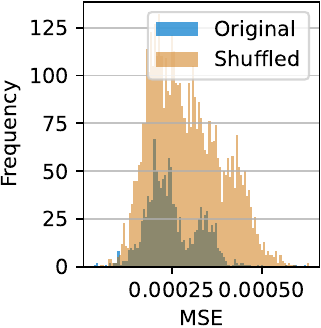}
    \vspace{-2em}
\end{minipage}\hfill
    \begin{minipage}[c]{0.42\linewidth}
\caption{\textbf{Histograms} of the mean squared errors (MSE) when fitting an AR process to sequences in $D$ (original, in blue) or $D_{\mathrm{shuffle}}$ (shuffled, in orange). We only display MSEs bigger than a threshold of $10^{-12}$.} \label{fig:hist}
    \end{minipage}
\end{figure}

\paragraph{Validation of the token encoding choice.}
Throughout the paper, we assume that the $s_{1:T}$ are generated following an AR process $s_{t+1} = Ws_t$. \camera{Even though we acknowledge that the AR process is an overly simplistic model for real-word sentences,} we provide empirical justification for using it by showing that such a process better explains real data than random ones. 
We use the \texttt{nltk} package \citep{bird2009natural}, and we employ classic literary works, specifically 'Moby Dick' by Herman Melville sourced from Project Gutenberg. 
We use the tokenizer and word embedding layer of a pre-trained GPT-2 model \citep{radford2019language}, and end up with about $325 000$ token representations
in dimension $d=1280$, that we reformat in a dataset $D$ of shape $(n, T, d)$, where $T=5$ (we keep the relative order of each token).
We also consider a shuffled counterpart of $D$ where the shuffling is done across the first two dimensions.
In other words, we create a dataset $D_{\mathrm{shuffle}}$ from a permutation of the tokens of the book.
\\
We then fit AR processes for each sequence in the two datasets using loss \eqref{eq:inner}, which we minimize by solving a linear system.
It should be noted that the problem remains non-trivial for some sequences, despite \( T \) being significantly smaller than \( d \). This complexity arises because certain sequences might contain identical elements with differing successors or predecessors. 
We hypothesize that when sequences are shuffled,  the number of such inconsistencies increases since the language's structure is lost.
This hypothesis is validated in Figure \ref{fig:hist}, where we display the histograms of the fitting losses when they are bigger than $10^{-12}.$
There are $4$ times more sequences with such an error for the shuffled dataset than for the original. This shows that the AR process is better suited when data present some semantics.
\begin{figure}[H]
    \begin{minipage}[c]{0.7 \linewidth}
\includegraphics[width=\textwidth]{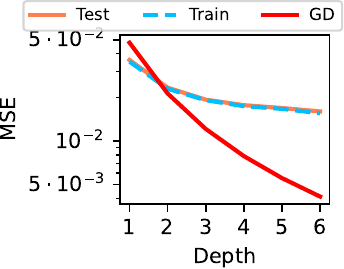}
    \end{minipage}\hfill
    \begin{minipage}[c]{0.3\linewidth}
\caption{\textbf{Evolution} of the MSE with depth $L$. 
We compare with $L$ steps of gradient descent on the inner loss \eqref{eq:inner}. At initialization, the MSE is between $1$ and $2$.} \label{fig:depth}
    \end{minipage}
\end{figure}
\paragraph{Augmented setting.}

We investigate whether the results of \S \ref{sec:aug} still hold without assumptions \ref{asp:commute} and \ref{asp:init}.
We consider the model \eqref{eq:augmented_att} on the augmented tokens $e_t=(0, s_{t}, s_{t-1})$.
We iterate relation \eqref{eq:augmented_att} with several layers, using layer normalization \citep{ba2016layer}. We consider depth values from $1$ to $6$. We generate a dataset with $n = 2^{14}$ sequences with $T_{\mathrm{max}}=50$ and $d=5$ (therefore $e_t \in \RR^{15}$) for training. We test using another dataset with $2^{10}$ sequences of the same shape. We train for $2000$ epochs with Adam \citep{kingma2014adam} and a learning rate of $5 \times 10^{-3}$ to minimize the mean squared error (MSE)
$
\min \ell (\Theta) \eqdef \sum_{T=2}^{T_{\mathrm{max}}} \frac1n \sum_{i=1}^n \| \mathcal{T}^{L}_{\Theta}(e^{i}_{1:T}) - s^{i}_T \|^2$, where $\mathcal{T}^{L}_{\theta}$ correspond to $L$ layers of \eqref{eq:augmented_att} 
 (we apply the forward rule $L$ times, and then consider the section of first $d$ coordinates). 
We compare the error with $L$ steps of gradient descent on the inner loss \eqref{eq:inner}, with a step size carefully chosen to obtain the fastest decrease. 
We find out that even though the first Transformer layers are competitive with gradient descent, the latter outperforms the Transformer by order of magnitudes when $L \geq 3$. 
Results are displayed in Figure \ref{fig:depth}. 
The fact that several steps of gradient descent outperform the same number of $\mathcal{T}_{\theta}$ layers is not surprising, as Proposition \ref{prop:gd} does not generalize to more than one layer.
\camera{In contrast, as shown in Appendix \ref{app:exp}, a full Transformer with all the bells and whistles as described in \citet{vaswani2017attention} ($\mathrm{softmax}$ and MLP applied component-wise to each Transformer layer) outperforms gradient descent and has a similar trend, as shown in Figure \ref{fig:rebuttal_augmented_general}.}
\paragraph{Non-Augmented setting.}
We now investigate whether the results of \S \ref{sec:non_aug} still hold without assumptions \ref{asp:commute} and \ref{asp:diag}. 
We consider the model $\mathcal{T}_{\theta}$ in \eqref{eq:att}. 
\begin{figure*}[ht]
 \centering
\includegraphics[width=1\textwidth]{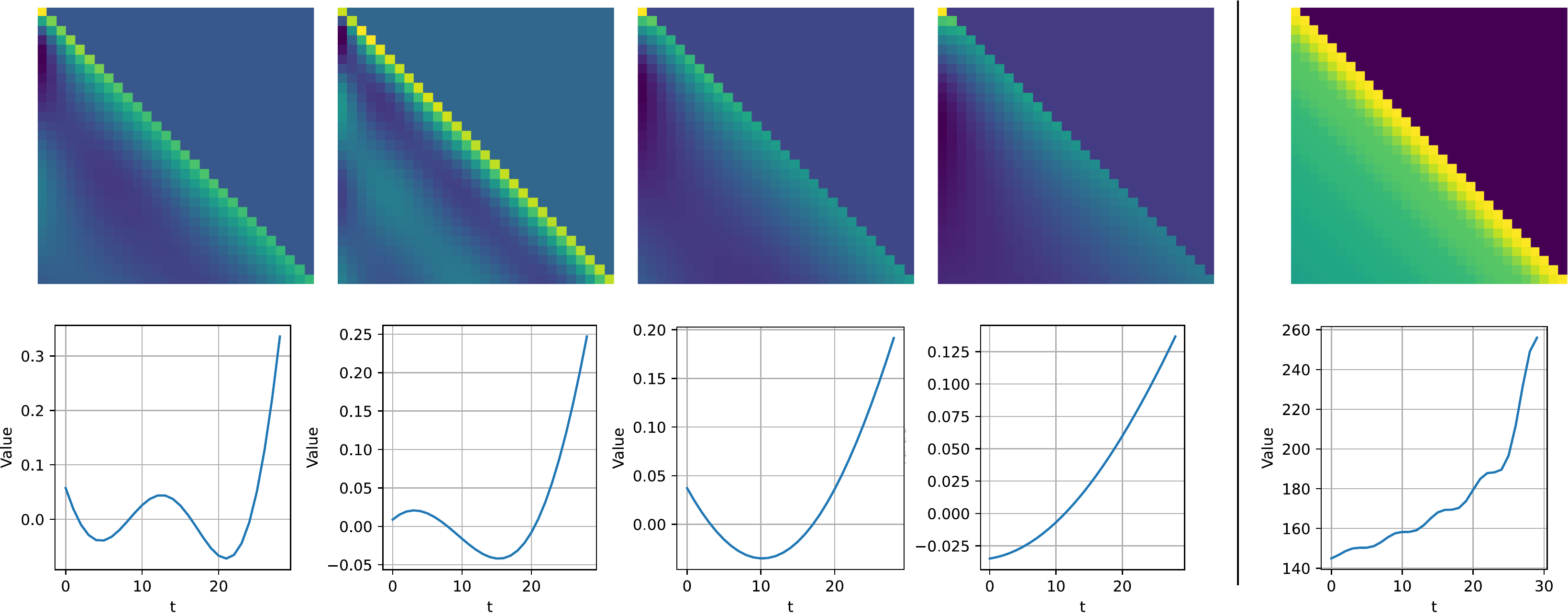} 
\vspace{-2em}
\caption{\textbf{Left:} Positional encodings after training for $\mu \in \{50, 100, 200, 300 \}$. The first raw corresponds to the matrix $P$, and the second raw to a plot of its last raw. \textbf{Right:} Comparison with the cosine absolute positional encoding standardly used in machine translation \cite{vaswani2017attention} (we display $p p^\top$). In both cases, we observe an invariance across diagonals. In addition, for high $\mu$ (i.e. small variations of the context), the most recent tokens are more informative, as imposed by the inductive prior of the cosine positional encoding.}
\label{fig:pe}
\vspace{-1em}
 \end{figure*}
We parameterize the positional encoding in the linear Transformer equation \eqref{eq:att} using the $\mathrm{softmax}$ of a positional attention-only similarity cost matrix with learnable  parameters $W_{Q_{\mathrm{pos}}}$ and $W_{K_{\mathrm{pos}}}$: 
$
P_{t,t'}= \mathrm{softmax}(\langle  W_{Q_{\mathrm{pos}}}p_t|W_{K_{\mathrm{pos}}}p_{t'}\rangle), 
$
as we found it to stabilize the training process. We use a similar dataset as in the previous section, i.e., a training set with $2^{14}$ sequences, each with $50$ elements of dimension $d=10$, and we test using another dataset with $2^{10}$ sequences of the same shape.

We train models $\mathcal{T}_{\theta}$ in \eqref{eq:att} for $200$ epochs with different numbers of heads. We use the Adam optimizer with a learning rate of $10^{-2}$. Without further modification, we do not observe a significant gain as the number of heads increases.
However, when duplicating the data along the dimension axis, that is $e_t \eqdef (s_t, s_t)$, we observe a significant improvement, as illustrated in Figure
\ref{fig:heads}. \camera{Understanding why duplicating the data leads to a significant improvement is left for future work.}
\begin{figure}[H]
    \begin{minipage}[c]{0.65\linewidth}
        \includegraphics[width=\textwidth]{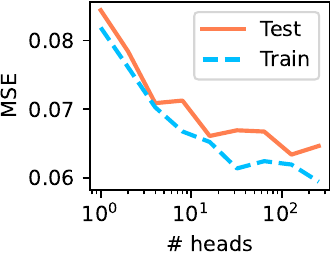}
    \end{minipage}\hfill
    \begin{minipage}[c]{0.3\linewidth}
\caption{\textbf{Evolution} of the MSE with the number of heads. At initialization, the MSE is between $0.35$ and $1$.} \label{fig:heads}
    \end{minipage}
\end{figure}
To further relate our experimental findings to our theory, we also exhibit an orthogonality property between heads after training. 
For this, we take $d=5$ to ease the visualization, and initialize each parameter equally across heads but add a small perturbation $.05 \times  \Nn(0,1)$ to ensure different gradient propagation during training. 
We then train the model and compare the quantity $(\sum_{h=1}^H (B^\top_h B_h))_{(i,j)}$ after training and at initialization. 
Note that it corresponds to a measure of orthogonality of heads.
Results are displayed in Figure~\ref{fig:VV}.
We observe that after training, an orthogonality property appears. 
In addition, as we are duplicating the tokens across dimensions, we can see that heads become specialized in attending to some coordinates across tokens.  
\begin{figure}[H]
\centering
\includegraphics[width=1.\columnwidth]{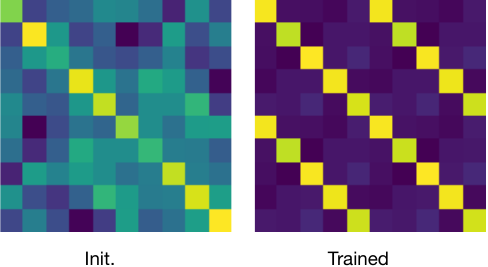} 
\caption{\textbf{Matrices } $(\sum_{h=1}^H (B^\top_h B_h))_{(i,j)} \in \RR^{10 \times 10}$ at initialization and after training. The trained parameters lead to an orthogonality between heads, as predicted by our theory.}\label{fig:VV}
\end{figure}

\paragraph{Change in the context distribution.}
We consider the setting of \S \ref{subsec:pe_only}, using the empirical loss counterpart of \eqref{eq:pe}, averaged over $T \in \{ 2, \cdots, T_{\mathrm{max}}\}$.
We generate a dataset with $10^4$ examples and $T_{\max} = 30$.
We train our positional encoding-only model with gradient descent and stop training (early stopping) when the loss is smaller than $10^{-3}$. We initialize $P_{t,t'} = 0$. Results are in Figure \ref{fig:pe}, where we mask coefficients $P_{T-1,T}$ (which are close to $1$ after training) in the display to investigate the behavior of the extra coefficients. We observe that the trained positional encoding exhibits an invariance across diagonals. Importantly, each row $P_t$ has a smooth behavior with $t'$, that we compare to absolute cosine positional encodings \citep{vaswani2017attention}.

\section*{Conclusion}
In this work, we study the in-context autoregressive learning abilities of linear Transformers to learn autoregressive processes of the form $s_{t+1} = W s_t$. In-context autoregressive learning is decomposed into two steps: estimation of $W$ with an in-context map $\Gamma$, followed by a prediction map $\psi$. Under commutativity and parameter structure assumptions, we first characterized $\Gamma$ and $\psi$ on augmented tokens, in which case $\Gamma$ is a step of gradient descent on an inner objective function. We also considered non-augmented tokens and showed that $\Gamma$ corresponds to a non-trivial geometric relation between tokens, enabled by an orthogonality between trained heads and learnable positional encoding. 
We also studied positional encoding-only attention and showed that approximate solutions of minimum $\ell_2$ norm are favored by the optimization.
Moving beyond commutativity assumptions, we extended our theoretical findings to the general case through experiments.
\vspace{-1em}
\paragraph{Future work.}
Investigating the case where $\tau = T$ in the non-augmented setting would lead to approximated in-context mappings, where achieving zero loss is no longer possible. This investigation would provide further insight into the role of positional encoding in estimating $W$ in-context. \camera{Another investigation left for future work is to consider the case of non-commuting context matrices and to relate the computation of a Transformer to a proxy for a gradient flow for estimating $W$ in-context, using the connection between Transformers and gradient flows \cite{sander2022sinkformers}.}

\vspace{-1em}
\section*{Impact Statement}
This paper is mainly theoretical and its goal is to advance our understanding of Transformers. 
\vspace{-1em}
\section*{Acknowledgments}
The work of M. Sander and G. Peyré was supported by the
French government under the management of Agence Nationale de la Recherche as part of the ``Investissements
d’avenir'' program, reference ANR-19-P3IA-0001 (PRAIRIE 3IA Institute). 
 T. Suzuki was partially supported by JSPS KAKENHI (24K02905) and JST CREST (JPMJCR2115, JPMJCR2015).
MS thanks Scott Pesme and Francisco Andrade for fruitful discussions.
\vspace{-1em}
\bibliographystyle{icml2024}
\bibliography{sample}
\newpage
\newpage
\appendix
\onecolumn

\section{Proofs}

In what follows, as we consider complex numbers, we use the hermitian product over $\CC^d,$ that is $\langle \alpha, \beta\rangle  \eqdef \beta^\star \alpha =\sum_i \alpha_i \Bar{\beta_i}$. 

We denote $\Uu^d = \{ \lambda \in \CC^d | |\lambda_i|=1 \forall i\in \{1, \cdots, d \} \}.$

\subsection{Proof of Lemma \ref{lem:augmented_heads}.}\label{proof:augmented_heads}
\begin{proof}
Let, for $s \in \RR^d$, $W^1_V s \eqdef (0, s, 0) \in \RR^{3d}$ and $W^2_V s \eqdef (0, 0, s)$. 

We now simply consider a positional attention-only model, that is, for $h \in \{1,2\}$: $$\mathcal{A}^{h}_{t,:}=\mathrm{softmax}(P^h_{t, :}).$$
We can choose the positional encodings $P^1$ and $P^2$ such that ${A}^{1}_{t,t'} \simeq \delta_{t'=t}$ and ${A}^{2}_{t,t'} \simeq\delta_{t'=t-1}.$

Then 
$$
\sum_{h=1}^{2} \sum_{t'=1}^{t} \mathcal{A}^h_{t,t'} W^h_V s_{t'} \simeq (0, s_t, s_{t-1}) = e_t. 
$$
\end{proof}

\subsection{Proof of Proposition \ref{prop:gd_von}.}\label{proof:gd_von}

\begin{proof}
    
We briefly recall the reasoning as presented in \citet{vonoswald2023uncovering}, and consider the case $W_0 = 0$ for simplicity. See \citet{vonoswald2023uncovering} for a full proof.
Let \begin{equation}\label{eq:opt_params}
\scalebox{0.8}{ %
\(
A = \left(\begin{array}{ccc}
0 & 0 & 0 \\
0 & 0 & 0 \\
0 & I_d & 0
\end{array}\right) \quad \mathrm{and} \quad B = \left(\begin{array}{ccc}
0 & \eta I_d & 0 \\
0 & 0 & 0 \\
0 & 0 & 0
\end{array}\right).
\)
}
\end{equation}

Then the section vector of the first $d$ coordinates in \eqref{eq:augmented_att} is
$
\eta \sum_{t=1}^{T} s_{t}s^\top_{t-1} s_T.
$

The gradient of $L$ at $W_0 = 0$ is:
$$
\nabla_W L (0,e_{1:T}) = -\sum_{t=1}^{T-1} (s_{t+1})s^\top_{t}.
$$
Therefore, \eqref{eq:augmented_att} corresponds to a single step of gradient descent starting from $W_0 = 0$.
\end{proof}

\subsection{Proof of Proposition \ref{prop:gd}.}\label{proof:gd}

In what follows, the sums over $t$ are from $t=1$ to $T$ and the sums over $j$ are from $j=1$ to $d$.

We first consider the following Lemma.
\begin{lem}\label{lem:loss_augmented}
Under assumptions \ref{asp:commute} and \ref{asp:init}, the loss of the linear Transformer writes:
\begin{equation}\label{eq:loss_augmented}
    \ell((a_i), (b_i))= \EE_{\lambda}{\sum_{T=2}^{T_\mathrm{max}} \sum_{i=1}^d |\sum_{t,j,\alpha \in \Aa, \beta \in \Bb} c_{\alpha, \beta}\lambda^{T-t-\alpha}_j\lambda^{t-1+\beta}_i - \lambda^{T}_i|^2}
\end{equation}
with $\Aa = \{-1, 0, 1\}$ and $\Bb = \{-1,0\}.$ We have $c_{\alpha, \beta} = u_{\alpha} v_{\beta} $ for $u_{0} = a_1 + a_4$, $u_{1} = a_2$, $u_{-1} = a_3$, $v_{0} = b_1$ and $v_{-1} = b_2$.
\end{lem}

\begin{proof}
    
One has $Ae_T = (0, A_1 s_T + A_2 s_{T-1},  A_3 s_T + A_4 s_{T-1}).$ Therefore one has $\langle Ae_{T}, e_t \rangle = s_t^\star A_1 s_T + s_{t}^\star A_2 s_{T-1} +  s_{t-1}^\star A_3 s_T + s_{t-1}^\star A_4 s_{T-1}.$ Since $(Be_t)_{1:d} = B_1 s_t + B_2 s_{t-1},$
one obtains through \eqref{eq:augmented_att}:
$$
\mathcal{T}_{\theta}(e_{1:T}) = \sum_t (a_1 s_t^\star s_T + a_2 s_{t}^\star s_{T-1} +  a_3 s_{t-1}^\star s_T + a_4 s_{t-1}^\star s_{T-1})(b_1 s_t + b_2 s_{t-1}).
$$
Developing, we obtain
$$
\mathcal{T}_{\theta}(e_{1:T})_i  = \sum_{t,j,\alpha \in \Aa}u_{\alpha}\lambda^{T-t-\alpha}_j \sum_{\beta \in \Bb} v_{\beta}\lambda^{t-1+\beta}_i = \sum_{t,j,\alpha \in \Aa, \beta \in \Bb}c_{\alpha, \beta}\lambda^{T-t-\alpha}_j\lambda^{t-1+\beta}_i,
$$
which implies the result.
\end{proof}

Using the notations of Lemma \ref{lem:loss_augmented}, Proposition \ref{prop:gd} now writes as follows.

\begin{prop}[In-context autoregressive learning with gradient-descent.]\label{prop:gd_app}
Suppose  $\Cc = \Cc_U$, assumptions \ref{asp:commute} and \ref{asp:init}.
Then loss \eqref{eq:loss_augmented} is minimal for $c_{\alpha, \beta} = 0$ if $(\alpha, \beta) \neq (-1, 0)$ and $c_{-1,0} = \frac{\sum_{T=2}^{T_{\mathrm{max}}}T}{\sum_{T=2}^{T_{\mathrm{max}}}(T^2 + (d-1)T)}$. 
Therefore, the optimal in-context map $\Gamma_{\theta^*}$ is one step of gradient descent starting from the initialization $\lambda = 0$, with a step size asymptotically equivalent to $\frac{3}{2 T_{\max}}$ with respect to $T_{\mathrm{max}}$.
\end{prop}

\begin{proof}
We develop the term in the sum in \eqref{eq:loss_augmented}:
\begin{equation*}
\begin{split}
&|\sum_{t,j,\alpha \in \Aa, \beta \in \Bb} c_{\alpha, \beta}\lambda^{T-t-\alpha}_j\lambda^{t+\beta-1}_i - \lambda^{T}_i|^2  \\ &= \sum_{t_1,j_1,\alpha_1, \beta_1, t_2,j_2,\alpha_2, \beta_2} c_{\alpha_1, \beta_1}c_{\alpha_2, \beta_2}\lambda^{T-t_1-\alpha_1}_{j_1}\lambda^{-T+t_2+\alpha_2}_{j_2}\lambda^{t_1 - t_2+\beta_1 - \beta_2}_i \\ &- 2 \mathrm{Real}(\sum_{t,j,\alpha \in \Aa, \beta \in \Bb} c_{\alpha, \beta}\lambda^{T-t-\alpha}_j\lambda^{t+\beta-1}_i \lambda_i^{-T}) + 1.
\end{split}
\end{equation*}
We now need to compute the expectations of the first two terms in the sum. Because $\EE_{\lambda}{(\lambda^k_i)} = \delta_{k=0}$  and the $\lambda_i$ are i.i.d., most terms will be zeros. 
\begin{itemize}
    \item For the first term, one needs to look at the different possible values for $(j_1, j_2, i)$ to calculate $\EE_{\lambda}(\lambda^{T-t_1-\alpha_1}_{j_1}\lambda^{-T+t_2+\alpha_2}_{j_2}\lambda^{t_1 - t_2+\beta_1 - \beta_2}_i)$.
    
    - 1) If $j_1 = j_2 = i$, it is $\EE_{\lambda}(\lambda_i^{\alpha_2 - \alpha_1 + \beta_1 - \beta_2}) = \delta_{\alpha_1 - \alpha_2 = \beta_1 - \beta_2}$. 
    
    - 2) If $j_1 = j_2 \eqdef j \neq i$, it is $\EE_{\lambda}(\lambda_j^{t_2 - t_1 + \alpha_2 - \alpha_1})  \EE_{\lambda}(\lambda_i^{t_1 - t_2 + \beta_1 - \beta_2})= \delta_{t_2 - t_1 = \alpha_1 - \alpha_2 = \beta_1 - \beta_2}$.
    
    - 3) If $j_1 \neq j_2$, $i \neq j_1$, $i \neq j_2$, then
    $\EE_{\lambda}(\lambda^{T-t_1-\alpha_1}_{j_1}\lambda^{-T+t_2+\alpha_2}_{j_2}\lambda^{t_1 - t_2+\beta_1 - \beta_2}_i) = \delta_{t_1-t_2 = \beta_2 - \beta_1, T-t_1-\alpha_1=0, T- t_2 - \alpha_2 = 0}$.

    - 4) If $j_1 \neq j_2$ and $i=j_1$, the expectation is $\delta_{T-\alpha_1-t_2+\beta_1 - \beta_2 = 0, T = t_2 + \alpha_2}$, and similarly when $j_1 \neq j_2$ and $i=j_2$.
    
    As a consequence, we see that all the terms that do not satisfy $\alpha_1 - \alpha_2 = \beta_1 - \beta_2$ will lead to $0$ expectation, which therefore implies that the first term writes:
    $$\EE_{\lambda}(\sum_{t_1,j_1,\alpha_1, \beta_1, t_2,j_2,\alpha_2, \beta_2, \alpha_1 - \alpha_2 = \beta_1 - \beta_2} c_{\alpha_1, \beta_1}c_{\alpha_2, \beta_2}\lambda^{T-t_1-\alpha_1}_{j_1}\lambda^{-T+t_2+\alpha_2}_{j_2}\lambda^{t_1 - t_2+\beta_1 - \beta_2}_i).$$
\item For the second term, we have 
$$
\sum_{t,j,\alpha \in \Aa, \beta \in \Bb} c_{\alpha, \beta}\lambda^{T-t-\alpha}_j\lambda^{t+\beta-1}_i \lambda_i^{-T} = \sum_{t,\alpha, \beta} c_{\alpha, \beta}\sum_j(\lambda^{T-t-\alpha}_j\lambda^{t+\beta-1-T}_i).
$$
- When $i \neq j$, the expectation of 
$
\lambda^{T-t-\alpha}_j\lambda^{t+\beta-1-T}_i
$
is not $0$ if and only if $t+\alpha = T$ and $T = t + \beta - 1$. Given that $\Aa = \{-1, 0, 1\}$ and $\Bb = \{-1,0\}$, this implies $\alpha = \beta - 1$  and therefore $\alpha=-1$ and $\beta = 0$. But then we have $t = T+1$ which is not possible. 

- When $i=j$, the expectation of 
$
\lambda^{T-t-\alpha}_j\lambda^{t+\beta-1-T}_i = \lambda_i^{\beta - \alpha - 1}
$
is not $0$ if and only if $\beta = \alpha + 1$, that is $\beta=0$ and $\alpha = -1$.

Therefore, one has 
$$
\EE_{\lambda}(\sum_{t,\alpha, \beta} c_{\alpha, \beta}\sum_j(\lambda^{T-t-\alpha}_j\lambda^{t+\beta-1-T}_i)) = \sum_{t} c_{-1,0} = T c_{-1,0}
$$
and the second term is $-2T c_{-1,0}$.
\end{itemize}
Back to the full expectation, isolating the term in $c^2_{-1,0}$ from the first term, we get 
\begin{equation}\label{eq:expectation}
    \EE_{\lambda}|\sum_{t,j,\alpha \in \Aa, \beta \in \Bb} c_{\alpha, \beta}\lambda^{T-1-t-\alpha}_j\lambda^{t+\beta-1}_i - \lambda^{T}_i|^2 = 
\end{equation}
$$ \EE_{\lambda}(\sum_{t_1,j_1,\alpha_1, \beta_1, t_2,j_2,\alpha_2, \beta_2, \alpha_1 - \alpha_2 = \beta_1 - \beta_2, (\alpha_1, \alpha_2, \beta_1, \beta_2) \neq (-1,-1,0,0)} c_{\alpha_1, \beta_1}c_{\alpha_2, \beta_2}\lambda^{T-t_1-\alpha_1}_{j_1}\lambda^{-T+t_2+\alpha_2}_{j_2}\lambda^{t_1 - t_2+\beta_1 - \beta_2}_i) $$
$$+ K_T c^2_{-1,0} -2T c_{-1,0} + 1,$$
for some constant $K_T$.

We now examine the terms in the sum within the expectation. We want to show that $E_1 = E_2$ where
$$E_1 \eqdef \{ (\alpha_1, \alpha_2, \beta_1, \beta_2) |\alpha_1 - \alpha_2 = \beta_1 - \beta_2, (\alpha_1, \alpha_2, \beta_1, \beta_2) \neq (-1,-1,0,0)\} 
$$
and
$$E_2 \eqdef \{ (\alpha_1, \alpha_2, \beta_1, \beta_2) |\alpha_1 - \alpha_2 = \beta_1 - \beta_2, (\alpha_1, \beta_1) \neq (-1, 0), (\alpha_2, \beta_2) \neq (-1, 0)\} .$$ We already have $E_2 \subset E_1$. If $(\alpha_1, \alpha_2, \beta_1, \beta_2) \in E_1 \backslash E_2$, then  either $(\alpha_1, \beta_1) = (-1, 0)$ or $(\alpha_2, \beta_2) = (-1, 0)$.
If $(\alpha_1, \beta_1) = (-1, 0)$, since $\alpha_1 - \alpha_2 = \beta_1 - \beta_2$, then $\alpha_2 = \beta_2 - 1$, which necessarily implies $\beta_2 = 0$ and $\alpha_2 = -1$, which contradicts the fact that $(\alpha_1, \alpha_2, \beta_1, \beta_2) \in E_1.$
Similarly, if $(\alpha_2, \beta_2) = (-1, 0)$ and $\alpha_1 - \alpha_2 = \beta_1 - \beta_2$, then $\beta_1 = 0$ and $\alpha_1 = -1$.
Therefore, $E_1 = E_2$, and 
$$ \EE_{\lambda}(\sum_{t_1,j_1, t_2,j_2, (\alpha_1, \alpha_2, \beta_1, \beta_2) \in E_1} c_{\alpha_1, \beta_1}c_{\alpha_2, \beta_2}\lambda^{T-t_1-\alpha_1}_{j_1}\lambda^{-T+t_2+\alpha_2}_{j_2}\lambda^{t_1 - t_2+\beta_1 - \beta_2}_i) = $$
$$ \EE_{\lambda}(\sum_{t_1,j_1, t_2,j_2, (\alpha_1, \alpha_2, \beta_1, \beta_2) \in E_2} c_{\alpha_1, \beta_1}c_{\alpha_2, \beta_2}\lambda^{T-t_1-\alpha_1}_{j_1}\lambda^{-T+t_2+\alpha_2}_{j_2}\lambda^{t_1 - t_2+\beta_1 - \beta_2}_i) =$$
$$ \EE_{\lambda}(|\sum_{t,j,\alpha \in \Aa, \beta \in \Bb, (\alpha, \beta) \neq (-1, 0)} c_{\alpha, \beta}\lambda^{T-t-\alpha}_j\lambda^{t+\beta-1}_i|^2) \geq 0.$$

We are interested in the minimum of \eqref{eq:expectation}. The minimum of $K_T c^2_{-1,0} -2T c_{-1,0} + 1$ is reached for $c_{-1, 0} \neq 0$.
We also just showed that the first term is non-negative, so the minimum will be reached when (almost surely in $\lambda \sim \Uu(\Cc_U)$) 
$$
\sum_{t,j,\alpha \in \Aa, \beta \in \Bb, (\alpha, \beta) \neq (-1, 0)} c_{\alpha, \beta}\lambda^{T-t-\alpha}_j\lambda^{t+\beta-1}_i = 0.
$$
In particular, the terms corresponding to monomials in $\lambda_i$ are
$$
\sum_{t,\alpha \in \Aa, \beta \in \Bb, (\alpha, \beta) \neq (-1, 0)} c_{\alpha, \beta}\lambda^{T -\alpha + \beta-1}_i = (T \sum_{\alpha \in \Aa, \beta \in \Bb, (\alpha, \beta) \neq (-1, 0)} c_{\alpha, \beta} )\lambda^{T -\alpha + \beta-1}_i = 0.
$$
Therefore, identifying the coefficients of this polynomial gives $c_{1, -1} = c_{-1, -1} + c_{0,0} = c_{1,0} + c_{0,-1} = 0$. We want to show that this implies $v_{-1} = u_1 = u_{0} = 0$. 

If $v_{-1} \neq 0$, then because $c_{1, -1} = 0$, we have $u_1 = 0$. From $c_{1,0} + c_{0,-1} = 0$, it follows $c_{0,-1} = 0$ and therefore $u_0 = 0$. From $c_{-1, -1} + c_{0,0} = 0$, this implies $c_{-1, -1} = 0$ and thus $u_{-1} = 0$, which contradicts $c_{-1,0} \neq 0$. Therefore, $v_{-1} = 0$. 

Now, if $u_1 \neq 0$, then $v_{-1} = 0$, and from $c_{-1, -1} + c_{0,0} = 0$, we have $c_{0,0} = 0$. But because $c_{-1,0} \neq 0$, we have $u_0 = 0$, which combined with $c_{1,0} + c_{0,-1} = 0$ implies $c_{1,0} = 0$, which is impossible because $v_0 \neq 0$. Therefore, $u_1 =0$.

So $u_1 = v_{-1} = 0$, and $c_{0,0} = 0$, so that $u_0 = 0$. This shows that loss \eqref{eq:loss_augmented} is minimal for $c_{\alpha, \beta} = 0$ if $(\alpha, \beta) \neq (-1, 0)$. 

Last, we need to calculate the constant $$ K_T \eqdef \EE_{\lambda}(\sum_{t_1,j_1, t_2,j_2} \lambda^{T-t_1+1}_{j_1}\lambda^{-T+t_2-1}_{j_2}\lambda^{t_1 - t_2}_i).$$
Back to the different possible values for $(j_1, j_2, i)$ analyzed above, we get non-zeros when $j_1 = j_2 = i$ and when $j_1 = j_2 \neq i$. In cases 3) and 4) we obtain as a necessary condition for non-zero expectation: $t_2 = T+1$ or $t_1 = T+1$; which is not possible. 
Therefore,
 $$ K_T =( \sum_{j_1 = j_2 = i}  \sum_{t_1, t_2} 1) + (\sum_{j_1 = j_2 \neq i}  \sum_{t_1 = t_2} 1) = T^2 +(d-1)T.$$
We now denote $\eta = a_3b_1= c_{-1,0}$. 
 Replacing the zero terms in loss \eqref{eq:loss_augmented}, we obtain
$$
d \sum^{T_\mathrm{max}}_{T=2} \left( \eta ^2(T^2 + (d-1)T)  - 2 \eta T + 1\right),
$$
for which the argmin is given by 
$$\eta^* = \frac{\sum_{T=2}^{T_{\mathrm{max}}}T}{\sum_{T=2}^{T_{\mathrm{max}}}(T^2 + (d-1)T)} \sim \frac{3}{2T_{\mathrm{max}}} \quad \text{as} \quad T_{\mathrm{max}} \to +\infty.$$

Finally at optimality,
$$
\mathcal{T}_{\theta^*}(e_{1:T}) = \sum_t \eta^* s_{t-1}^\star s_T s_t = \eta^* (\sum_t s_t s_{t-1}^\star) s_T = - \eta^* \nabla_W L(0,e_{1:T})s_T = \Gamma_{\theta^*}(e_{1:T})s_T
$$
with $\Gamma_{\theta^*}(e_{1:T}) \eqdef - \eta^* \nabla_W L(0, e_{1:T}).$
\end{proof}

\subsection{Proof of Lemma \ref{lem:forward_heads}.}\label{proof:forward_heads}
\begin{proof}
For a given input sequence $e_{1:T} = s_{1:T} = (s, \lambda \odot s, \cdots, \lambda^{T-1} \odot s)$, with context $\lambda \in \Uu^d$, one has
$$
\mathcal{T}_{\theta}(e_{1:T}) = \sum_{t=1}^TP_{T-1,t} \sum_{h=1}^H \langle  \lambda^{t-1} \odot s , a_h \odot \lambda^{T-2}\odot s\rangle  b_h \odot \lambda^{t-1} \odot s.
$$
We have 
$$
(\langle \lambda^{t-1} \odot s , a_h \odot \lambda^{T-2}\odot s\rangle  b_h \odot \lambda^{t-1} \odot s)_i = \sum_{j=1}^d {a^j_h \lambda^{2-T}_j \lambda^{t-1}_j b^i_h \lambda^{t-1}_i}.
$$
Since we precisely have 
$$
([\B^\top \A] \lambda^{t-T+1} \odot \lambda^{t-1})_i = \sum_{j=1}^d \sum_{h=1}^H  {b^i_h a^j_h \lambda^{t-T+1}_j \ \lambda^{t-1}_i},
$$
this gives us the desired result.

\end{proof}
\paragraph{Remark.} \camera{In fact, looking at the above proof, considering arbitrary real values for the vector $s_0=s$, one can absorb the terms in $s^2_j$ in each $a^j_h$ and the terms in $s_i$ in each $b^i_h$, since these are learnable parameters. Therefore, our results can be adapted to any arbitrary initial value $s_0$.}
\subsection{Proof of Proposition \ref{prop:optimality_unitary}.}\label{proof:optimality_unitary}
\begin{proof}
    
Denote $\C = \B^\top \A$. Let us suppose that we have an optimal solution $\theta$ such that $l(\theta) = 0.$ Therefore, for almost all $\lambda \in \Uu^d$ and $s_t = \lambda^{t-1}$, one has $\mathcal{T}_{\theta}(s_{1:T})= \lambda^T$. Then, $\forall i \in \{ 1, \cdots, d\}$:
$$ \sum^{T}_{t=1} P_{T-1, t} \sum_{j=1}^d \C_{ij} \lambda^{t-T+1}_j \lambda^{t-1}_i  = \lambda_i^T.$$ 
By identifying the coefficients of the polynomial in the $\lambda_i$'s, we see that one must have for all $T \geq 2$ that $P_{T-1,t}=0$ if $t \neq T$, and for all $1 \leq i \leq d$, $p_{T-1,T} \C_{ii} = 1$, and $\C_{ij} = 0$ for $i\neq j.$

An optimal in-context mapping is then obtained by considering the forward rule of $\mathcal{T}_{\theta^*}$ for optimal parameters. One gets:
$$
\mathcal{T}_{\theta^*}(e_{1:T}) = (\bar{e}_{T-1} \odot e_T) \odot(e_T).
$$

\end{proof}

\subsection{Proof of Proposition \ref{prop:quadra_loss}.}\label{proof:quadra_loss}

\begin{proof}
The loss writes
$$
\ell(\theta) = \sum_{T=2}^{T_{\mathrm{max}}} \EE_{\diag{\lambda} \sim \mathcal{W}} \| \sum^{T}_{t=1}P_{T-1,t} \C \lambda^{t-T+1} \odot \lambda^{t-1} - \lambda^T\|^2.
$$
We compute the expectation of each term in the sum by first developing it. 

One has
$$
\| \sum^{T}_{t=1}P_{T-1,t} \C \lambda^{t-T+1} \odot \lambda^{t-1} - \lambda^T\|^2 = \sum_{t, t'}P_{T-1,t}P_{T-1,t'}\sum_{i,j,k}\C_{ij}\C_{ik} \lambda_j^{T-t-1}\lambda_k^{t'-T+1}\lambda_i^{t'-t}
$$
$$
- 2 \mathrm{Real}(\sum_t P_{T-1,t} \sum_{i,j}\C_{ij}\lambda_j^{T-t-1}\lambda_i^{T-t+1}) + d.
$$
Looking at the expectation of the first term
$$
\sum_{t, t'}P_{T-1,t}P_{T-1,t'}\sum_{i,j,k}\C_{ij}\C_{ik} \EE(\lambda_j^{T-t-1}\lambda_k^{t'-T+1}\lambda_i^{t'-t}),$$
we see that one has to calculate for $t, t', i, j, k$ 
$$
\EE(\lambda_j^{T-t-1}\lambda_k^{t'-T+1}\lambda_i^{t'-t}).
$$

When $j\neq k$, it is $\delta_{t'=t=T-1}.$ When $j=k$ it is $\delta_{t'=t}.$

Therefore 
$$\sum_{t, t'}P_{T-1,t}P_{T-1,t'}\sum_{i,j,k}\C_{ij}\C_{ik} \EE(\lambda_j^{T-t-1}\lambda_k^{t'-T+1}\lambda_i^{t'-t}) = \sum_{t=t'}P^2_{T-1,t}\sum_{i,j}\C^2_{ij} + P^2_{T-1,T-1}\sum_{i,j,k}\C_{i,j} \C_{i,k}=$$

$$\|P_{T-1}\|^2 \|\C \|_F^2 + P^2_{T-1, T-1}S(\C^\top \C).$$

Similarly, because $\EE(\lambda_j^{T-t-1}\lambda_i^{T-t+1}) = \delta_{t=T, i=j}$, the second term is 
$$
- 2 \EE(\mathrm{Real}(\sum_t P_{T-1,t} \sum_{i,j}\C_{ij}\lambda_j^{T-t-1}\lambda_i^{T-t+1})) = -2 P_{T-1, T} \mathrm{Tr}(\C).
$$
This concludes the proof.
\end{proof}

\subsection{Proof of Proposition \ref{prop:1d}.}\label{proof:1d}

\begin{proof}
Let us denote $a(t)$, $b(t)$ and $p(t)$ the functions defined by the gradient flow on loss $\ell$, that is $\dot{a} = -\nabla_a \ell (a,b,p)$ and similarly for $b$ and $p$.
We start with the result from \citet{nguegnang2021convergence, achour2021loss}, which states that \(a(t)\), \(b(t)\), and \(p(t)\) are bounded and converge to a stationary point \((a^*, b^*, p^*)\) as \(t \to \infty\): 
\[
\lim_{t \to \infty} (a(t), b(t), p(t)) = (a^*, b^*, p^*).
\]

The possible stationary points are either global or non-global minima. The conditions for these are:
\begin{itemize}
    \item Global minimum if \( p^* a^* b^* = 1 \)
    \item Non-global minimum if \( a^* b^* = 0 \), \( a^* p^* = 0 \) and \( b^* p^* = 0 \) and therefore 
\[
p^* a^* b^* = 0.
\]
\end{itemize}

If the system were to converge to a non-global minimum, we would therefore have 
\[
\ell(+\infty) = |p^* a^* b^* - 1| = |0 - 1| = 1 > \ell(0)
\]
This is a contradiction because the energy should have decreased over time. 

As a result, the only remaining possibility for the stationary points is that they must satisfy \(p^* a^* b^* = 1\). Therefore, the functions \(a(t)\), \(b(t)\), and \(p(t)\) must converge to a global minimum.
\end{proof}

\subsection{Proof of Lemma \ref{lem:trig}.}\label{proof:trig}

\begin{proof}
For the announced parameters, the problem simply decomposes into sub-problems in dimension $2$. Indeed, we have for all $i\in \{1, \dots, \delta \}$: 
$$
(\mathcal{T}_{\theta^*}(e_{1:T}))_{2i-1} = -\frac{1}{2} ( \lambda_{2i-1}^{0} + \bar{\lambda}_{2i-1}^{0})\lambda_{2i-1}^{T-2} +  ( \lambda_{2i-1}^{1} + \bar{\lambda}_{2i-1}^{1})\lambda_{2i-1}^{T-1} = \lambda_{2i-1}^{T-2} (-1 +  \lambda_{2i-1}^{2} + 1) =  \lambda_{2i-1}^T.
$$

Similarly, 
$$
(\mathcal{T}_{\theta^*}(e_{1:T}))_{2i} = \lambda_{2i}^T.
$$
\end{proof}

\subsection{Proof of Proposition \ref{prop:optimality_ortho}.}\label{proof:optimality_ortho}
\begin{proof}
    The proof is similar to Proposition \ref{proof:optimality_unitary}, by regrouping each $\lambda_i$ with $\bar{\lambda}_i$ and identifying coefficients in two polynomials.
    More precisely, one must have

    $$ \sum^{T}_{t=1} P_{T-1, t} \sum_{j=1}^d \C_{2i-1,j} \lambda^{t-T+1}_j \lambda^{t-1}_{2i-1}  = \lambda_{2i-1}^T.$$ 

    Therefore, isolating terms in $\lambda_{2i-1}$ (recall that $\lambda_{2i} = 1/{\lambda_{2i-1}}$), and developing, we get, noting $p \eqdef P_{T-1}$:

    $$
    \C_{2i-1,2i-1}(\sum_{t < T-1}p_{t}\lambda^{2t-T}_{2i-1} + p_{T-1}\lambda^{T-2}_{2i-1} + p_T\lambda^{T}_{2i-1}) + \sum_t p_t \C_{2i-1, 2i}\lambda^{T-2}_{2i-1} = \lambda^{T}_{2i-1}.
    $$

Identifying gives $(\C_{2i-1,2i-1} + \C_{2i-1,2i})p_{T-1} + \C_{2i-1,2i}p_T = 0$, $\C_{2i-1,2i-1}p_T = 1$ and $p_{t<T-1}=0$.

Similarly, on the conjugates, $(\C_{2i,2i} + \C_{2i,2i-1})p_{T-1} + \C_{2i,2i-1}p_T = 0$, $\C_{2i,2i}p_T = 1$.

Identifying the other terms gives $C_{2i,j} = C_{2i-1,j} = 0$ for $j \neq 2i$ and $j\neq2i-1$. 
\end{proof}

\camera{\paragraph{Interpretation.} Up to rescaling, the relation $p_T C_{i,i} = 1$ gives $p_T = C_{i,i} = 1$, and therefore
\begin{equation}\label{eq:conj}
    C_{2i-1, 2i} + (1+C_{2i-1,2i})p_{T-1} = C_{2i, 2i-1} + (1+C_{2i,2i-1})p_{T-1} = 0,
\end{equation}
which gives $C_{2i-1, 2i} = C_{2i,2i-1}$. Therefore, $C$ is symmetric and has the form $\diag(J_b, \dots, J_b)$ with $J_b = ((1,b),(b,1))$ and $b = -p_{T-1} / (1 + p_{T-1})$.
If $H < d$, $C$ cannot be full rank, and therefore necessarily $C_{2i,2i-1} = 1$ or $-1$. But $C_{2i,2i-1} = -1$ is impossible given \eqref{eq:conj}. We then recover Lemma \ref{lem:trig}.
}

\subsection{Proof of Proposition \ref{prop:pe}}\label{proof:pe}
\begin{proof}
    One has that the second order term in the quadratic form \eqref{eq:pe} writes 
    $$
    \langle p | H p \rangle = \sum_{t,t'} p_tp_{t'}\EE_{\lambda \sim \Ww(\mu)}\lambda^{2(t'-t)}.
    $$
    We can calculate this expectation in closed form. 
    Writing $\alpha = \frac{2\pi}{\mu}$, it gives 
    $$
    \EE_{\lambda \sim \Ww(\mu)}\lambda^{2(t'-t)} = \frac{1}{\alpha}\int_{0}^{\alpha}e^{2i\theta(t'-t)}d\theta,
    $$
which gives 
$\frac{i}{2 \alpha(t'-t)}[1 - e^{2i(t'-t)\alpha}]$ if $t \neq t'$ and $1$ otherwise. Since $\langle p | H p \rangle$ is real, we can identify the real parts so that 
$$
H_{t, t’} = (\mu/ (4 \pi (t’-t)) )\sin(4(t’-t)\frac{\pi}{\mu}).
$$
We now turn to the eigenvalues of $H$. $H$ is a smooth function of $\alpha$:
$$H : [0, 2 \pi] \to S_T,$$ 
where $S_T$ are the symmetric matrices of $\RR^{T\times T}$. Using Th. 5.2 from \citet{kato2013perturbation}, we know that the eigenvalues of $H$ can be parametrized as continuous functions $\nu_1(\alpha) \geq \nu_2(\alpha) \geq \cdots \nu_T(\alpha).$ Since for $\alpha = 0$, the eigenvalues of $H$ are $T, 0, \cdots 0$, and recalling that $\mu = \frac{2\pi}{\alpha}$, we obtain the result.

\end{proof}

\section{Additional Experiments}\label{app:exp}

\paragraph{Effect of the addition of $\mathrm{softmax}$ layers and MLP layers on the trained solutions.}
We consider the unitary context matrix setup of section \ref{sec:non_aug}. As mentionned, we could not find a natural way to express the global minimum of the training loss when a $\mathrm{softmax}$ layer was involved, even in dimension 1. To get more insight, we conducted an additional experiment where we trained different models with and without $\mathrm{softmax}$ and MLP layers. We used $d=1$, and $T=10$, and a hidden dimension of 32 in the MLP. Results are shown in Figure \ref{fig:rebuttal_non_augmented_commuting}, where it is clear that in the case of commuting context matrices, using a $\mathrm{softmax}$ is incompatible with learning the underlying in-context mapping.

\begin{figure}[h]
    \centering
    \includegraphics{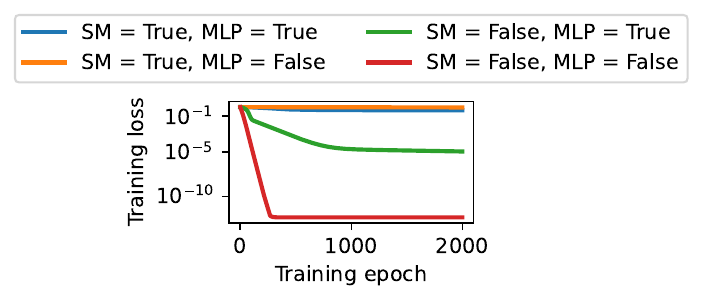}
    \caption{Training loss with training epoch for different configurations of the Transformer's architecture, when using a $\mathrm{softmax}$ (SM) and an MLP or not.}
    \label{fig:rebuttal_non_augmented_commuting}
\end{figure}

\paragraph{Transformer with all the bells and whistles and gradient descent.}

The experiment displayed in Figure \ref{fig:depth} shows that linear attention fails to compete with gradient descent when there are more than $2$ layers. In contrast, a full Transformer with all the bells and whistles as described in \citet{vaswani2017attention} ($\mathrm{softmax}$ and MLP applied component-wise to each transformer layer) outperforms gradient descent and has a similar trend, as shown in Figure \ref{fig:rebuttal_augmented_general}. The training procedure and the dataset are identical to those described in the \textit{Augmented setting} paragraph of Section \ref{sec:exp}. 

\begin{figure}[h]
    \centering
    \includegraphics{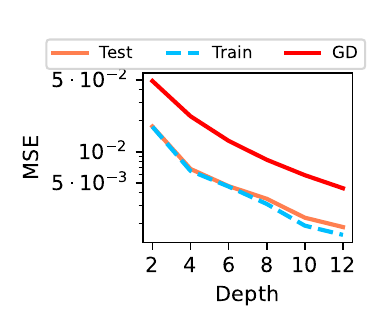}
    \caption{\textbf{Evolution} of the mean squared errors (MSE) with depth $L$ for a vanilla Transformer  \cite{vaswani2017attention}. 
We compare with $L$ steps of gradient descent (GD) on the inner loss \eqref{eq:inner}. At initialization, the MSE is between $1$ and $2$. .}
    \label{fig:rebuttal_augmented_general}
\end{figure}

\newpage

\end{document}